\documentclass[letterpaper, 10 pt, conference]{ieeeconf}  % Comment this line out if you need a4paper

\IEEEoverridecommandlockouts                              % This command is only needed if 
\usepackage{graphics} % for pdf, bitmapped graphics files
\usepackage{epsfig} % for postscript graphics files
\usepackage{times} % assumes new font selection scheme installed
\usepackage{amsmath} % assumes amsmath package installed
\usepackage{amssymb}  % assumes amsmath package installed
\usepackage{booktabs}
\usepackage{multicol}
\usepackage{multirow}
\usepackage{makecell}
\usepackage{tabularx}
\usepackage{array}
\usepackage[bookmarks=true]{hyperref}

\usepackage{cite}
\usepackage{algorithm, algpseudocode}
\let\labelindent\relax
\usepackage{enumitem}

\DeclareMathOperator*{\argmax}{arg\,max}
\DeclareMathOperator*{\argmin}{arg\,min}
\newtheorem{definition}{Definition}

\newtheorem{lemma}{Lemma}
\newtheorem{theorem}{Theorem}
\newtheorem{assumption}{Assumption}

\newtheorem{corollary}{Corollary}

\title{\LARGE \bf
    Value-Preserving Action Quantization
}

\author{Max Muchen Sun, Jueun Kwon, Benjamin Riviere, and Todd Murphey
\thanks{Max Muchen Sun, Jueun Kwon and Todd Murphey are with the Department of Mechanical Engineering, Northwestern University, Evanston, IL 60208, USA. Email: {\tt\small msun@u.northwestern.edu}.}
\thanks{Benjamin Riviere is with the Department of Mechanical and Aerospace Engineering and the Department of Computer Science, New York University, New York City, NY 10012, USA}
\thanks{This work is supported by ARO grant W911NF-22-1-0286. The views expressed are the authors’ and not necessarily those of the funders.}
}

\begin{document}

\allowdisplaybreaks

\maketitle

\thispagestyle{empty}
\pagestyle{empty}

%%%%%%%%%%%%%%%%%%%%%%%%%%%%%%%%%%%%%%%%%%%%%%%%%%%%%%%%%%%%%%%%%%%%%%%%%%%%%%%%

\begin{abstract}

Search-based policy learning, such as Monte Carlo tree search (MCTS), offers the flexibility to trade runtime computation for data efficiency and adaptation to novel environments. For robot learning, however, applying search requires discretizing the continuous action space and associated dynamics, typically through action quantization. While reconstruction-based quantization is sufficient for representing the continuous action space alone, policy learning additionally requires the quantized actions to preserve their effects on future states and values. We show that such requirement is not necessarily captured by action reconstruction error. Instead, we prove that the error in the value function resulting from action quantization is tightly bounded by the reconstruction error of state reachability under the quantized actions. Based on these results, we propose value-preserving action quantization (VAQ), which quantizes actions by preserving state reachability. We further combine VAQ with MCTS (VAQ-MCTS) as a value-preserving, search-based policy learning framework for behavioral cloning. Using a wide range of object pushing tasks as a controlled benchmark, we evaluate VAQ-MCTS against state-of-the-art behavioral cloning methods. VAQ-MCTS matches or exceeds their performance under standard settings, while substantially improving data efficiency and robustness to heterogeneous demonstrations, while enabling zero-shot adaptation to novel environments. Hardware experiments further demonstrate the effectiveness of VAQ-MCTS, completing all 30 real-world object pushing trials with only 10 minutes of demonstration data.

\end{abstract}

%%%%%%%%%%%%%%%%%%%%%%%%%%%%%%%%%%%%%%%%%%%%%%%%%%%%%%%%%%%%%%%%%%%%%%%%%%%%%%%%

\section{Introduction}

Search-based policy learning offers an alternative to relying entirely on offline policy learning by allocating computation at runtime~\cite{silver_mastering_2016, grill_monte-carlo_2020, castellini_scalable_2023}. For robot learning, this means that a robot can use limited offline behavioral data while adapting its decisions to the current task and environment. This tradeoff is particularly attractive when collecting large-scale demonstrations is expensive and deployment conditions often differ from those encountered during training. 

However, applying search to robot learning problems remains challenging because the robot's continuous action space needs to be discretized, and the corresponding branching factor grows rapidly with the dimensionality and resolution of the action space. Action quantization---the construction of a discrete, low-dimensional representation of continuous actions~\cite{van_den_oord_neural_2017}---offers a potential solution and has been increasingly used in robot learning~\cite{hubert_learning_2021, dadashi_continuous_2022, luo_action-quantized_2023, lee_behavior_2024, wu_discrete_2025}. By replacing the continuous action space with a compact set of well-represented discrete actions, action quantization can make runtime learning with online search computationally feasible with minimal compromise in performance.

However, action quantization is not simply action discretization. For policy search, the quantized actions must represent not only the action space, but also the consequences of actions. For robot learning in particular, nonlinear dynamics could skew how actions affect state evolution and, subsequently, the performance of the quantized actions. In this paper, we study the consequences of action quantization from the perspective of value preservation: \emph{How do quantized actions affect the performance of the optimal policy compared to the original continuous action space?} We formally show that the error in the optimal value function induced by action quantization is not necessarily determined by how well the quantized actions reconstruct the original action space, but by how well the resulting quantized dynamics reconstruct the reachable state set of the original system.

Based on this result, we propose value-preserving action quantization (VAQ), adapting vector-quantized variational autoencoders (VQ-VAEs)~\cite{van_den_oord_neural_2017} to preserve state reachability from behavioral data. We combine VAQ with Monte Carlo tree search (MCTS), a search algorithm that enables policy learning through online planning, to construct a behavioral cloning framework (VAQ-MCTS) that uses offline demonstration data to learn a compact action representation and performs online policy learning through runtime search. Compared to offline behavioral cloning methods, including those with action quantization, VAQ-MCTS can adapt zero-shot to novel environments through online planning while requiring limited offline data. Compared to MCTS with naive action quantization or discretization, VAQ maintains MCTS performance even under extremely small action budgets. An overview of the proposed framework is shown in Fig.~\ref{fig:overview}.

\begin{figure*}[ht!] 
    \centering    
    \includegraphics[width=\linewidth]{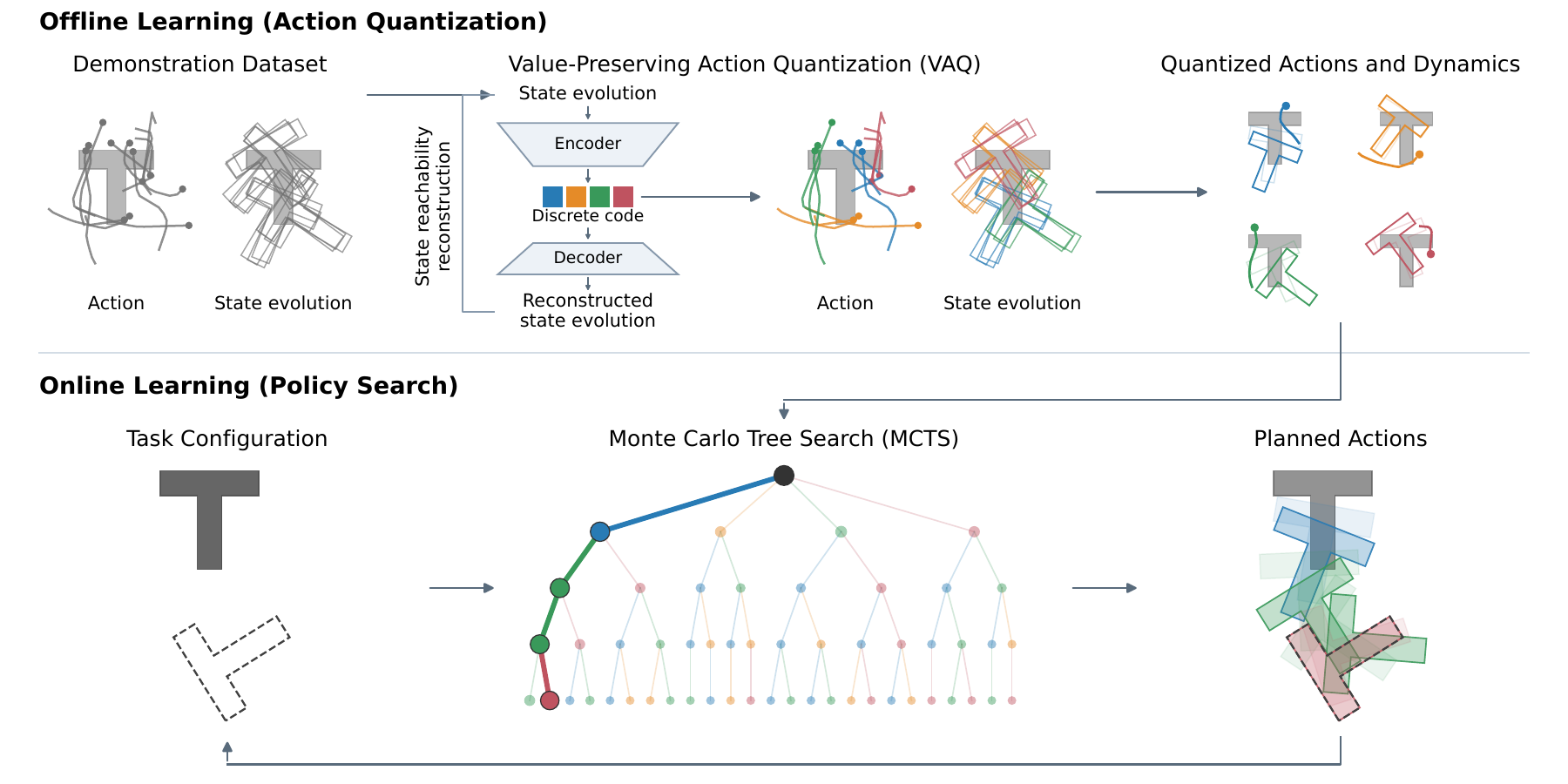}
    \vspace{-2em}
    \caption{Overview of VAQ-MCTS. VAQ learns a discrete latent representation that preserves state reachability; in this example, four discrete latent codes of VQ-VAE capture representative state evolutions. This compact representation then enables efficient runtime policy learning via MCTS planning.}
    \label{fig:overview}
    \vspace{-1em}
\end{figure*}

Through comprehensive simulated benchmark evaluations, we show that VAQ-MCTS matches or exceeds state-of-the-art behavioral cloning methods while improving data efficiency, adaptation to novel environments, and robustness to heterogeneous training data. On hardware, VAQ-MCTS achieves 100\% success across 30 trials and three manipulation tasks using only 10 minutes of demonstration data. We summarize our contributions in three points below:

\begin{enumerate}[leftmargin=*, label=(\arabic*)]
    \item We derive a value approximation bound for action quantization based on reachable-state reconstruction.
    \item We propose value-preserving action quantization (VAQ) that preserves state reachability.
    \item We combine VAQ with MCTS as an online policy learning framework via efficient runtime search.
\end{enumerate}

\section{Preliminaries and Related Work}

\subsection{Action Quantization}

We denote a Markov decision process (MDP) as $\mathcal{M} = (\mathcal{S}, \mathcal{A}, p, r, \gamma)$, where $\mathcal{S}$ is the continuous state space, $\mathcal{A}$ is the continuous action space, $p(s^\prime|s,a)$ is the state transition distribution, $r(s,a)$ is the reward function, and $\gamma \in [0,1)$ is the discount factor. An \emph{action quantizer} is defined as a mapping from the system state $s$ to a finite set of actions:
\begin{align}
    g_{\phi}(s) = \{ \bar{a}_1(s), \dots, \bar{a}_K(s) \}.
\end{align} 

\begin{definition}[Action quantization]
    Denote a behavioral dataset as 
    $\mathcal{D}=\{(s_i,a_i)\}_{i=1}^N$, where $s_i$ is the system state. Action quantization is defined as the optimization problem:
    \begin{gather}
        \phi^* = \argmin_{\phi} \sum_{i=1}^N D\Big(\mu_{\mathcal{D}}(\cdot|s_i), \nu_{g_{\phi}}(\cdot|s_i)\Big), \label{eq:aq_def} \\
        \mu_{\mathcal{D}}(a|s) = \frac{1}{M}
        \sum_{j=1}^N \mathbf{1}[s=s_j]\delta(a-a_j), \nonumber \\
        M = \sum_{j=1}^N \mathbf{1}[s=s_j], \text{ } \nu_{g_{\phi}}(a|s_i) = \frac{1}{K} \sum_{k=1}^K \delta\big(a-\bar a_k(s)\big), \nonumber
    \end{gather} where $\delta(\cdot)$ is the Dirac delta function and $\mathbf{1}[\cdot]$ is the indicator function. $\mu_{\mathcal{D}}(a|s)$ is the empirical action distribution, $\nu_{g_{\phi}}(a|s)$ is the quantized action distribution, and $D$ is a discrepancy measure between action distributions. 
\end{definition} 

Several specifications of the quantization cost are used in practice based on (\ref{eq:aq_def}), among which the vector quantization (VQ) cost~\cite{van_den_oord_neural_2017} is widely used:
\begin{gather}
    \phi^* = \argmin_{\phi} \sum_{i=1}^N \min_{k} \| a_i - \bar{a}_k(s_i) \|^2,\\
    g_{\phi}(s) = \{ \bar{a}_1(s), \dots, \bar{a}_K(s) \}. \label{eq:general_vq_cost}
\end{gather} 

\noindent\textbf{Related work.} Simple action quantization directly discretizes the original action space $\mathcal{A}$ using a discrete action set. However, direct discretization suffers from exponential complexity in the discretization resolution and non-uniform action utilization. Prior work has addressed these issues through factorized~\cite{tavakoli_action_2018, tang_discretizing_2020, seyde_solving_2022}, autoregressive~\cite{tessler_distributional_2019, chebotar_q-transformer_2023}, and graph-based~\cite{tavakoli_learning_2020} discretization schemes, as well as adaptive action quantization methods that learn state-dependent discrete action sets from data~\cite{dadashi_continuous_2022}\cite{luo_action-quantized_2023}. Existing approaches primarily use action discretization to accelerate policy learning. In contrast, this work studies action quantization as a value-preserving compression of the original continuous problem.

\begin{figure*}[ht!]
    \centering
    \includegraphics[width=1.0\linewidth]{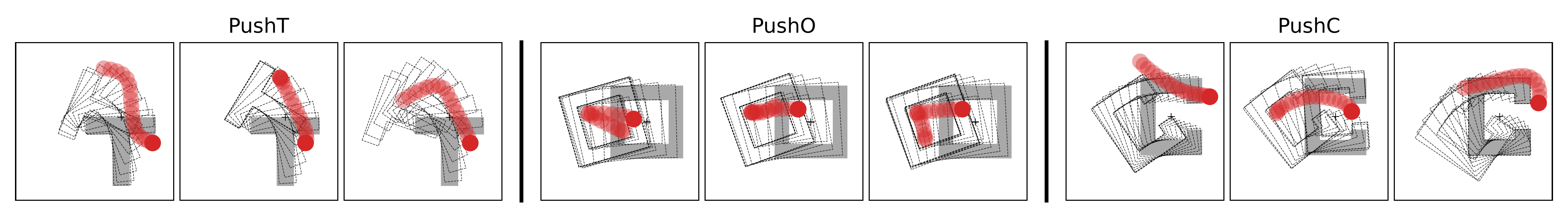}
    \vspace{-2em}
    \caption{Examples of action heterogeneity, where different action sequences (red end-effector motion) lead to similar state sequences. By Theorem~\ref{theorem:value_bound}, this mismatch can make action reconstruction-based quantization suboptimal for policy learning, motivating VAQ to reconstruct state reachability.}
    \label{fig:redundancy_example}
    \vspace{-1em}
\end{figure*}

\subsection{Vector-Quantized Variational Autoencoders} \label{sec:vqvae}

Vector-quantized variational autoencoders (VQ-VAEs)~\cite{van_den_oord_neural_2017} learn discrete latent representations from continuous data by inserting a vector-quantization layer into an autoencoder. A codebook $\mathbf{e}{=}\{e_1,\dots,e_K\}$ is a discrete set of latent vectors, in contrast to a continuous latent space. Given an encoder network $f_{\theta}(x)$, the quantizer $z_{\mathbf{e}}(x)$, which is parameterized by the codebook $\mathbf{e}$, maps the encoder output with its nearest codebook vector $z_{\mathbf{e}}(x) {=} \argmin_{e\in\mathbf{e}} \|f_{\theta}(x)-e\|$. The decoder network $\tilde{f}_{\psi}(z)$ then reconstructs the input from the quantized latent vector. 

The training of VQ-VAE jointly optimizes the encoder network $f_{\theta}$, decoder network $\tilde{f}_{\psi}$ and codebook $\mathbf{e}$ through:
\begin{align}
    \mathcal{L}_{\mathrm{VQ}}(x; \theta, \phi, \mathbf{e}) = & \|x-\tilde{f}_{\psi}(z_{\mathbf{e}}(x))\|^2 + \|\mathrm{sg} [f_{\theta}(x)] - z_{\mathbf{e}}(x)\|^2 \nonumber \\
    & + \beta \| f_{\theta}(x) - \mathrm{sg}[z_{\mathbf{e}}(x)] \|^2, \label{eq:vqvae_loss}
\end{align} where the first term is the reconstruction loss, while the last two terms update the codebook and encourage encoder outputs to remain close to their assigned codebook vectors, thereby clustering continuous latent representations into a finite set of discrete codes. The stop-gradient operator $sg[\cdot]$ controls which parameters are updated by each term despite the non-differentiable nearest-neighbor assignment.

\noindent\textbf{Related work.} VQ-VAE-based action quantization has been applied to both reinforcement learning and behavior cloning. In RL, state-conditioned VQ-VAE creates discrete-action problems in continuous settings, allowing policy constraints and conservative regularizers to be computed more precisely~\cite{luo_action-quantized_2023}. In behavior cloning, residual vector quantization has been used to tokenize continuous action chunks, enabling transformer policies to model multimodal behaviors while retaining efficient single-pass inference~\cite{lee_behavior_2024}. In this work, we use VQ-VAE-based action quantization to make online policy learning computationally feasible without compromising policy performance or optimality.

\subsection{Monte Carlo Tree Search}

Monte Carlo Tree Search (MCTS) is a search-based framework for planning over discrete decision spaces~\cite{kocsis_bandit_2006}\cite{swiechowski_monte_2023}. Similar to value-based reinforcement learning methods such as Q-learning, MCTS estimates action values and derives a policy by biasing action selection toward higher expected returns. MCTS performs lookahead search from the current state by repeatedly simulating candidate action sequences, expanding newly visited states, estimating trajectory returns, and backpropagating these estimates to update action values along the tree. The rollout policy is obtained from the root search statistics, typically by selecting the action with the highest visit count or estimated value. In this work, MCTS serves as the backbone of online learning, providing a principled anytime optimization framework over the quantized action space with guaranteed asymptotic optimality given sufficient computation. Action quantization reduces the branching factor of MCTS, making it computationally efficient without sacrificing decision quality.

\noindent\textbf{Related work.} MCTS has been central to combining planning and learning in game-playing~\cite{silver_mastering_2016}\cite{ silver_general_2018}\cite{schrittwieser_mastering_2020}, where decision-time search improves a learned policy and generates targets for subsequent learning. In continuous control, however, constructing an effective discrete action space remains challenging. Closest in spirit to our work is spectral expansion tree search~\cite{riviere_monte_2024}, which constructs a low-complexity discrete expansion from a known dynamics model to preserve planning-relevant structure, leading to a guaranteed bound on value approximation. Similar spectral-expansion ideas have also been applied to contact-rich manipulation~\cite{lathrop_model_2024}\cite{sigurdson_spectral_2026}. Our work shares this planning-preservation perspective, but instead learns quantized actions from data.

\begin{figure*}[ht!]
    \centering
    \includegraphics[width=1.0\linewidth]{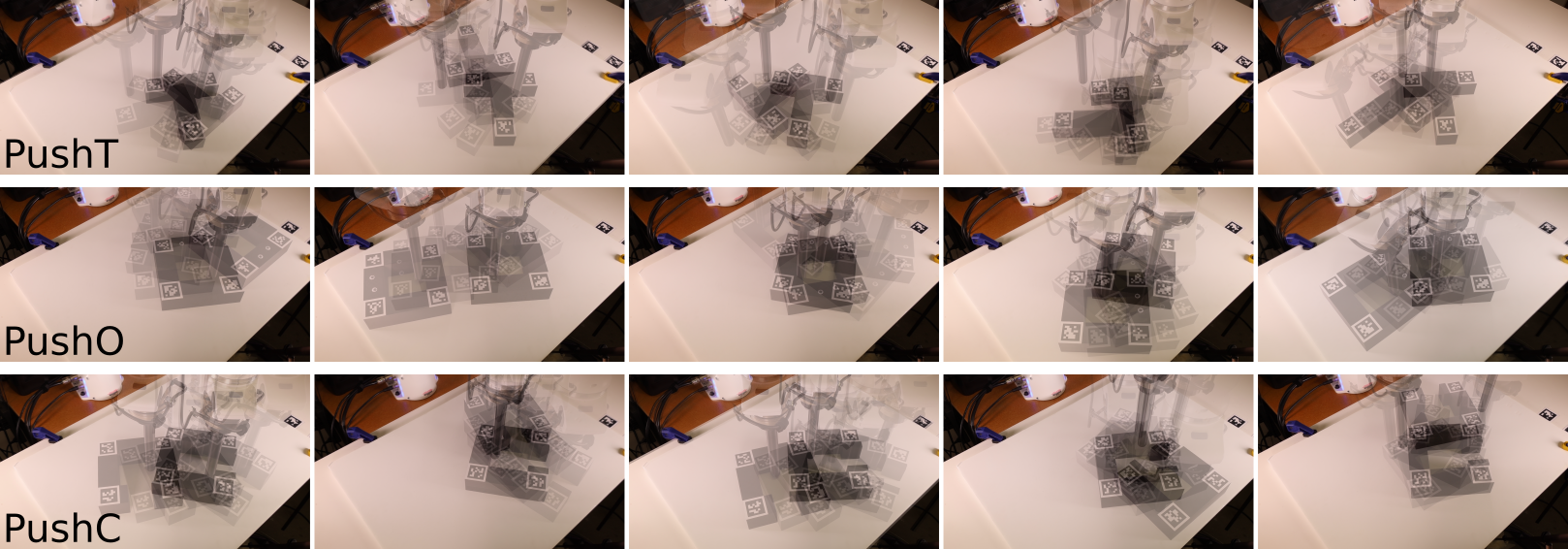}
    \vspace{-2em}
    \caption{VAQ extracts 8 actions from 10 minutes of undirected human demonstrations, achieving 100\% success across 30 trials and 3 objects.}
    \label{fig:vaq_franka}
    \vspace{-1em}
\end{figure*}

\section{Value-Preserving Action Quantization}

\subsection{The Value Approximation Bound of Action Quantization}

\begin{definition}[Optimal value function]
    For an MDP $\mathcal{M}=(\mathcal{S}, \mathcal{A}, p, r, \gamma)$, the optimal value function $V^*(s)$ represents the maximum expected reward achievable from a given state by any policy:
    \begin{align}
        V^*(s) &= \max_{\pi} V^{\pi}(s), \text{ } V^{\pi}(s) = \mathbb{E}\left[ \sum_{t=0}^{\infty} \gamma^t r(s_t, a_t) \Big| s_0 {=} s \right], \nonumber
    \end{align} where $V^{\pi}(s)$ is the value function of the given policy $\pi$. 
\end{definition}

\begin{definition}[Probabilistic reachable set]
    For an MDP $\mathcal{M}=(\mathcal{S}, \mathcal{A}, p, r, \gamma)$, we define:
    \begin{align}
        R_{\mathcal{M}}(s) = \{ p(\cdot|s,a) | a\in\mathcal{A} \}.
    \end{align} as the probabilistic reachable set at any given state $s\in\mathcal{S}$.
\end{definition}

\begin{definition}[Hausdorff distance]
    Given a domain $\mathcal{X}$, the Hausdorff distance between two sets $X, X^\prime\subseteq\mathcal{X}$ is:
    \begin{align}
        d_{\epsilon}(X, X^\prime) = \max\left\{ \sup_{X} \inf_{X^\prime} \epsilon(x, x^\prime), \sup_{X^\prime} \inf_{X} \epsilon(x, x^\prime) \right\}, \label{eq:hausdorff}
    \end{align} where $\epsilon(x,x^\prime)$ is a distance metric over $\mathcal{X}$.
\end{definition}

\begin{definition}[Probabilistic reachability bound]
    Consider two MDPs $\mathcal{M}_1=(\mathcal{S},\mathcal{A}_1,p,r,\gamma)$ and $\mathcal{M}_2=(\mathcal{S},\mathcal{A}_2,p,r,\gamma)$ that differ only in their action spaces. We define the probabilistic reachability bound as the supremum, over all states, of the $W_1$ Wasserstein Hausdorff distance between their probabilistic reachable sets:
    \begin{align}
        B(\mathcal{M}_1,\mathcal{M}_2) = \sup_{s\in\mathcal{S}} d_{W_1}\Big(R_{\mathcal{M}_1}(s),R_{\mathcal{M}_2}(s)\Big). \label{eq:prob_reach_bound}
    \end{align}
\end{definition}

\begin{theorem}[Value approximation bound]\label{theorem:value_bound}
    Consider two MDPs $\mathcal{M}_1=(\mathcal{S},\mathcal{A}_1,p,r,\gamma)$ and $\mathcal{M}_2=(\mathcal{S},\mathcal{A}_2,p,r,\gamma)$ that differ only in their action spaces, the difference between the optimal value functions of the two MDPs is bounded by:
    \begin{align}
        \| V^*_1(\cdot) {-} V^*_2(\cdot) \|_{\infty} \leq \frac{L_r + \gamma L_V}{1 - \gamma} \cdot B(\mathcal{M}_1,\mathcal{M}_2), \label{eq:value_bound}
    \end{align} where $L_r$ is the reward Lipschitz constant, $L_V {=} \max(L_{V^1}, L{V^2})$ is the larger optimal-value Lipschitz constant, and $d{W_1}$ is the $W_1$ Wasserstein Hausdorff distance between distribution sets.
\end{theorem}

Theorem~\ref{theorem:value_bound} (see Appendix for the proof) formalizes the intuition that action-space replacement is value-preserving only to the extent that it preserves the state-transition distributions that are reachable from each state. For learning-based action quantization, this implies that the relevant supervision is not from action reconstruction, but from reconstruction of the dynamics-dependent state reachable sets.

%%%%%%%%%%%%%%%%%%%%%%

\subsection{Action Reconstruction vs. Value Preservation}

\begin{corollary}
    Consider an action quantizer $g_\phi$ trained by minimizing the action reconstruction loss in~\eqref{eq:general_vq_cost}. If the transition model preserves action distances in the induced state-transition space, meaning that there exist constants $0<\alpha\leq\beta<\infty$ such that, $\forall s,a,a'$:
    \begin{align}
        \alpha \|a-a'\| \leq W_1\big(p(\cdot|s,a),p(\cdot|s,a')\big) \leq \beta \|a-a'\|,
    \end{align}
    then action reconstruction serves as a faithful proxy for the reachable-set discrepancy in Theorem~\ref{theorem:value_bound}. Therefore, the resulting quantizer is value-preserving.
\end{corollary}

This corollary follows directly from Theorem~\ref{theorem:value_bound}: under the distance-preserving condition, action-space reconstruction controls the reachable-set discrepancy that bounds the value approximation error. However, when the transition model distorts action distances, action reconstruction is no longer guaranteed to be value-aligned. One issue that is particularly relevant is heterogeneous demonstration data, where different demonstrations contain substantially different actions for achieving similar outcomes. In robotic systems, nonlinear dynamics can further deepen this mismatch: small action differences may induce large changes in state evolution (e.g., cart-pendulum, drone, and humanoid), while large action differences may induce similar state evolution (e.g., manipulation tasks, as illustrated in Fig.~\ref{fig:redundancy_example}). In such cases, action-space reconstruction can be misaligned with the value-relevant reachable-set discrepancy, such that quantized actions that are optimal for action reconstruction are suboptimal for value preservation.

%%%%%%%%%%%%%%%%%%%%%%

\begin{algorithm}[t!] 
\caption{VAQ-MCTS} \label{alg:vaq_mcts}
    \begin{algorithmic}[1] 
        \Require Dataset $\mathcal{D}$, codebook size $K$, horizon $H$
        \State Construct $H$-step dataset $\mathcal{D}_H$
        \State Learn state quantizer $h_\phi$ (\ref{eq:vaq_quantizer}) by minimizing (\ref{eq:vaq_vae_loss})
        \State Learn inverse dynamics $q(\mathbf{a}\mid s,\mathbf{s}^\prime)$
        \For{each decision step at state $s$}
            \State Run MCTS with branches $\bar{\mathbf{s}}^\prime_k \in h_\phi(s)$
            \State $k^* \gets$ best root branch returned by MCTS
            \State $\bar{\mathbf{a}} \sim q(\mathbf{a}\mid s,\bar{\mathbf{s}}^\prime_{k^*})$
            \State Execute $\bar{\mathbf{a}}$
        \EndFor
    \end{algorithmic}
\end{algorithm}

\begin{figure*}[ht!]
    \centering
    \includegraphics[width=\linewidth]{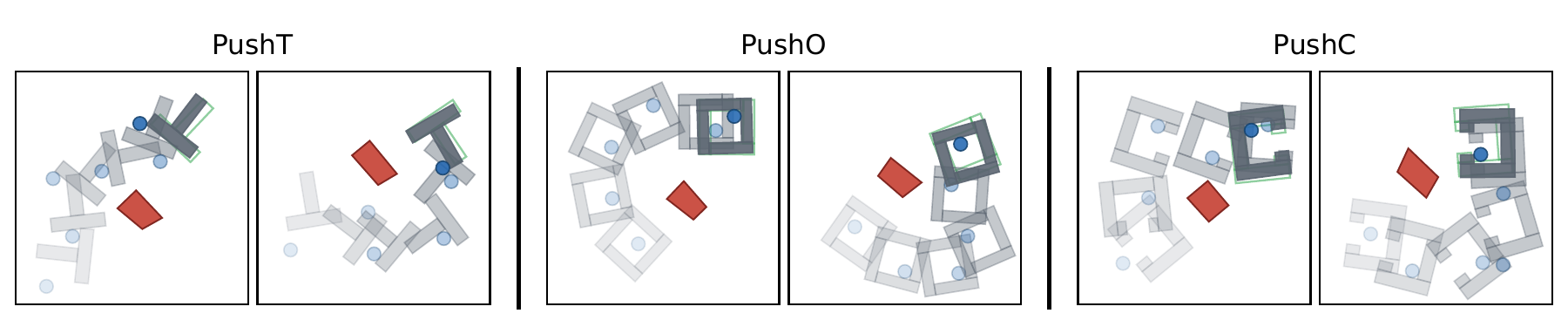}
    \vspace{-2.5em}
    \caption{Illustration of simulated benchmarks (PushT, PushO, and PushC). Robot pushes each object to the goal while adapting to unforeseen obstacles.}
    \label{fig:benchmark_illustration}
    \vspace{-1em}
\end{figure*}

\subsection{Value-Preserving Action Quantization (VAQ)} 
\label{subsec:vaq}

The above analysis suggests that, value preservation in action quantization requires supervision not based on reconstruction of the observed actions, but the reconstruction of the observed state reachable set from the dataset. 

Denote a dataset $\mathcal{D}=\{(s_i,a_i,s_i^\prime)\}_{i=1}^N$, where $s_i$ is the current state, $a_i$ is the current action, and $s_i^\prime$ is the next state. We first construct an $H$-step dataset $\mathcal{D}_{H}{=}\{(s_i,\mathbf{a}_i,\,\mathbf{s}_i^\prime)\}_{i=1}^M$, where $\mathbf{a}_i{=}[a_i, \dots, a_{i+H}]$ is an H-step action sequence and $\mathbf{s}_i^\prime{=}[s_i^\prime, \dots, s_{i+H}^\prime]$ is the resulting state evolution sequence by applying $\mathbf{a}_i$ at the state $s_i$. Note that, the $H$-step dataset becomes the standard next-state dataset when $H{=}1$. Next, we construct an $H$-step state quantizer denoted as $h_{\phi}(s) {=} \{ \bar{\mathbf{s}}^\prime_1(s), \dots, \bar{\mathbf{s}}^\prime_K(s) \}$, and optimize it with vector quantization cost over state evolution sequences:
\begin{gather}
    \phi^* = \argmin_{\phi} \sum_{i=1}^M \min_{k} \| \mathbf{s}^\prime_i - \bar{\mathbf{s}}^\prime_k(s_i) \|^2. \label{eq:state_vq_cost}
\end{gather} The state quantizer $h_{\phi}(s)$ can be implemented as a state-conditioned version of the VQ-VAE architecture introduced in Section~\ref{sec:vqvae}, including the encoder network $f_{\theta}(\mathbf{s},s)$, decoder network $\tilde{f}_{\psi}(z,s)$, the codebook $\mathbf{e}$, and the quantizer $z_{\mathbf{e}}(\mathbf{s},s) {=} \argmin_{e\in\mathbf{e}} \|f_{\theta}(\mathbf{s},s){-}e\|$, with the reconstruction term in the loss function (\ref{eq:vqvae_loss}) over state evolution:
\begin{align}
    \mathcal{L}_{\mathrm{VQ}}(\mathbf{s}, s; \theta, \psi, \mathbf{e}) = & \|\mathbf{s}-\tilde{f}_{\psi}(z_{\mathbf{e}}(\mathbf{s},s), s)\|^2 \nonumber \\
    & + \|\mathrm{sg} [f_{\theta}(\mathbf{s}, s)] - z_{\mathbf{e}}(\mathbf{s}, s)\|^2 \nonumber \\
    & + \beta \| f_{\theta}(\mathbf{s}, s) - \mathrm{sg}[z_{\mathbf{e}}(\mathbf{s}, s)] \|^2. \label{eq:vaq_vae_loss}
\end{align} After training, the decoder and codebook together constitute the state quantizer:
\begin{gather}
    h_{\phi}(s) = \{ \tilde{f}_{\psi^*}(e_k^*, s) \}_{k\in[1\dots K]}, \quad e_k^* \in \mathbf{e}^*, \label{eq:vaq_quantizer} \\
    \theta^*, \psi^*, \mathbf{e}^* = \argmin_{\theta, \psi, \mathbf{e}} \sum_{(\mathbf{s}^\prime_i, s_i)\in\mathcal{D}_H} \mathcal{L}_{\mathrm{VQ}}(\mathbf{s}^\prime_i, s_i; \theta, \psi, \mathbf{e}).
\end{gather} After learning the state quantizer, denote a learned inverse dynamics model $q(\mathbf{a} | s, \mathbf{s}^\prime)$, which generates an action sequence $\mathbf{a}$ given the current state $s$ and state sequence $\mathbf{s}^\prime$, the value-preserving action quantizer is:
\begin{gather}
    g(s) = \{ \bar{\mathbf{a}}_k(s)\}_{k\in[1,\dots,K]}, \text{ } \bar{\mathbf{a}}_k(s) \sim q(\mathbf{a} | s, \bar{\mathbf{s}}^\prime_k), \text{ } \bar{\mathbf{s}}^\prime_k \in h_{\phi}(s). \nonumber
\end{gather} The inverse dynamics model can be learned jointly with the state quantizer or separately from the dataset. In general, a generative model can be used. At the same time, because the state evolutions are quantized and do not require continuous interpolation, retrieval-based models provide another approach that enables faster inference~\cite{pari_surprising_2022, he_demystifying_2026, pfeifer_difference-aware_2026}.

\begin{corollary}
    Since vector quantization cost (\ref{eq:general_vq_cost}) is equivalent to the squared $2$-Wasserstein distance between the empirical distribution and the quantized distribution~\cite{canas_learning_2012}, the state quantization cost (\ref{eq:state_vq_cost}), as implemented by the reconstruction term in (\ref{eq:vaq_vae_loss}), preserves the state reachability.
\end{corollary}

\noindent\textbf{From VAQ to VAQ-MCTS } The learned state quantizer $h_{\phi}(s)$ generates a discrete set of the most representative state evolutions at any given state. This discrete set can therefore be directly used as the decision space for MCTS, where each tree branch corresponds to a quantized state evolution $\bar{\mathbf{s}}^\prime_k(s)$. MCTS searches over sequences of these quantized state evolutions to maximize the estimated return, and the selected state evolution is converted to an executable action sequence through the inverse dynamics model. An overview of the process is shown in Fig.~\ref{fig:overview} and described in Algorithm~\ref{alg:vaq_mcts}.

\begin{table*}[t]
    \centering
    \small
    \setlength{\tabcolsep}{7pt}
    \renewcommand{\arraystretch}{1.15}
    \begin{tabularx}{\textwidth}{l|*{3}{>{\centering\arraybackslash}X}|*{3}{>{\centering\arraybackslash}X}|*{3}{>{\centering\arraybackslash}X}}
    \toprule
    & \multicolumn{3}{c|}{\makecell{Original Environment\\(100 Demonstration; Obstacle-Free)}} & \multicolumn{3}{c|}{Obstacle Adaptation} & \multicolumn{3}{c}{10 Demonstrations} \\
    \cmidrule(lr){2-4} \cmidrule(lr){5-7} \cmidrule(lr){8-10}
    & PushT & PushO & PushC & PushT & PushO & PushC & PushT & PushO & PushC \\
    \midrule
    VAQ-MCTS (Ours) & 92\% & \textbf{100\%} & \textbf{100\%} & \textbf{78\%} & \textbf{97\%} & \textbf{81\%} & \textbf{89\%} & \textbf{98\%} & \textbf{91\%} \\
    Baseline-MCTS & \textbf{98\%} & 71\% & 98\% & 47\% & 59\% & 58\% & 68\% & 48\% & 6\% \\
    Diffusion Policy & 86\% & 95\% & 87\% & 8\% & 8\% & 12\% & 13\% & 86\% & 5\% \\
    VQ-BeT & 54\% & 84\% & 81\% & 6\% & 4\% & 4\% & 8\% & 54\% & 5\% \\
    SAQ-BC & 20\% & 11\% & 33\% & 2\% & 1\% & 4\% & 2\% & 38\% & 3\% \\
    SAQ-IQL & 21\% & 41\% & 45\% & 2\% & 7\% & 3\% & 2\% & 19\% & 5\% \\
    SAQ-CQL & 18\% & 11\% & 26\% & 0\% & 2\% & 4\% & 3\% & 17\% & 3\% \\
    SAQ-BRAC & 11\% & 18\% & 1\% & 0\% & 4\% & 0\% & 0\% & 5\% & 0\% \\
    \bottomrule
    \end{tabularx}
    
    \caption{Benchmark performance across standard tasks, obstacle adaptation, and limited demonstration data.}
    \label{tab:benchmark-results}
    \vspace{-2em}
\end{table*}

\section{Experiments}

\subsection{Experiment Overview}

\noindent\textbf{Task design } We focus on object pushing tasks in the benchmark evaluations. Object pushing is sufficiently complex for evaluating continuous control policies, as it involves reasoning over contact points, contact sequences, and contact actions. At the same time, planar pushing is interpretable and tractable enough to design task variants that deliberately test and ablate different aspects of action quantization. In particular, its tractability enables us to consistently generate heterogeneous demonstration data with distinct pushing strategies. In addition to the standard PushT task, we design two variants in which the robot manipulates O-shaped and C-shaped objects (Fig.~\ref{fig:benchmark_illustration}). These geometries support more diverse contact actions, thereby increasing the action heterogeneity in the demonstrations. We manually collect data for each task and purposefully include multiple strategies.

\noindent\textbf{Baseline selection } We include Diffusion Policy~\cite{chi_diffusion_2024} as an offline behavioral cloning baseline. We also include two action-quantized offline policy learning frameworks: vector-quantized behavior transformer (VQ-BeT)~\cite{lee_behavior_2024}, a behavioral cloning method, and state-conditioned action quantization (SAQ)~\cite{luo_action-quantized_2023}, including behavioral cloning (SAQ-BC) and three reinforcement learning variants: conservative Q-learning (SAQ-CQL), implicit Q-learning (SAQ-IQL), and behavior regularized actor-critic (SAQ-BRAC). Both VQ-BeT and SAQ use VQ-VAE for quantization but optimize action reconstruction. Lastly, as an ablation, we implement an MCTS baseline (Baseline-MCTS) using the identical VQ-VAE architecture and MCTS configuration as VAQ-MCTS, but with action reconstruction, instead of state reachability reconstruction, in the loss function (\ref{eq:vaq_vae_loss}).

\noindent\textbf{Implementation details } We implement the VQ-VAE using 16-step trajectory windows, with MLP encoder and decoder networks and a 16-dimensional latent space. State evolutions are represented in the object body frame and transformed back to the world frame during planning and execution. We use a retrieval-based inverse dynamics model for integration with MCTS, learning a dictionary that associates action sequences with quantized state evolutions. MCTS uses a branching factor equal to the number of VQ-VAE codes (default of 8), depth of 5, and 4096 simulations per planning step. MCTS does not use privileged dynamics information or an environment simulator.

\noindent\textbf{Experiment design } We organize the benchmark around the following questions:
\textbf{Q1.} How does VAQ-MCTS perform across tasks in the original environment, under novel environmental changes, and with limited demonstration data compared with other baselines?
\textbf{Q2.} Under the same VQ-VAE architecture and MCTS framework, does the value-preserving quantization loss in (\ref{eq:state_vq_cost}) better preserves optimal value function compared with action reconstruction-based quantization?
\textbf{Q3.} Under the same MCTS framework for online policy learning, does VAQ perform better under an extremely limited quantization budget?
\textbf{Q4.} How does action heterogeneity in demonstration data affect different action quantization paradigms: VAQ, VQ-BeT, and SAQ?

\begin{figure}[t!]
    \centering
    \includegraphics[width=1.0\linewidth]{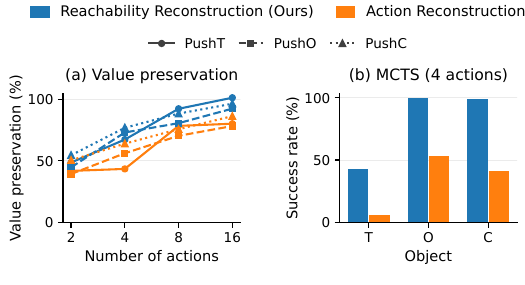}
    \vspace{-2.5em}
    \caption{Ablation on value preservation with reachability reconstruction.}
    \label{fig:ablation_plot}

    \vspace{1em}
    
    \centering
    \includegraphics[width=1.0\linewidth]{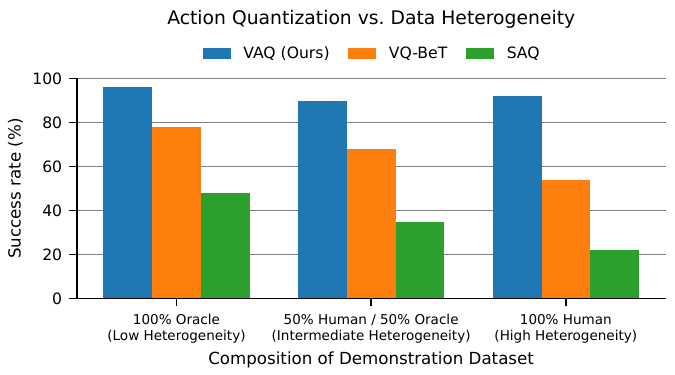}
    \vspace{-2em}
    \caption{Effects of demonstration heterogeneity on action quantization.}
    \label{fig:heterogeneity_plot}
    \vspace{-1em}
\end{figure}

\subsection{Experimental Results}

For \textbf{Q1}, we train VAQ-MCTS and Baseline-MCTS with 8 quantized latent codes. For the standard environments, all methods are trained with 100 obstacle-free demonstrations. As shown in Table~\ref{tab:benchmark-results}, VAQ-MCTS is competitive with Diffusion Policy and Baseline-MCTS in the original environments. Under obstacle adaptation and the 10-demonstration setting, VAQ-MCTS significantly outperforms all baselines across tasks, demonstrating the benefit of MCTS-based online policy learning for adaptation and data efficiency. While Baseline-MCTS performs similarly in the original environments, its performance degrades substantially when adapting to obstacles and learning from limited data, showing that online search is effective only when quantized actions preserve their effects on future states.

For \textbf{Q2}, we directly evaluate value preservation independent of the policy-learning method. At randomly sampled states, we approximate the optimal value using exhaustive tree search over a dense set of actions from the demonstration dataset. We then compare this value with those obtained from quantized actions learned using the state reachability reconstruction loss (\ref{eq:vaq_vae_loss}) and the action reconstruction loss under the same VQ-VAE architecture. We evaluate quantization budgets of 2, 4, 8, and 16 actions. As shown in Fig.~\ref{fig:ablation_plot}(a), state reachability reconstruction yields more accurate value estimates and converges to the optimal value faster, empirically validating the value-preservation bound.

For \textbf{Q3}, we limit both VAQ-MCTS and Baseline-MCTS to 4 quantized latent codes. As shown in Fig.~\ref{fig:ablation_plot}(b), VAQ-MCTS consistently and significantly outperforms Baseline-MCTS across tasks. In particular, for PushO and PushC, reducing the number of latent codes from 8 to 4 introduces minimal performance loss compared with Table~\ref{tab:benchmark-results}.

For \textbf{Q4}, we vary demonstration heterogeneity by mixing our manually collected PushT dataset with the oracle-generated dataset from~\cite{chi_diffusion_2024} at different ratios, and evaluate the three learning-based action quantization paradigms: VAQ, VQ-BeT, and SAQ. As shown in Fig.~\ref{fig:heterogeneity_plot}, the performance of both VQ-BeT and SAQ decreases as demonstration heterogeneity increases, while VAQ remains largely unaffected. This shows that value-preserving quantization is more robust to action heterogeneity in demonstration data than action reconstruction-based quantization.

\subsection{Hardware validation} We validate VAQ-MCTS on a Franka robot across the physical PushT, PushO, and PushC tasks, using only 10 minutes of undirected human play data for each object, in which the demonstrator does not complete the tasks. VAQ extracts 8 quantized latent codes, which VAQ-MCTS then uses for online planning to search for optimal policies that manipulate the objects from randomized initial states to their goal poses. VAQ-MCTS succeeds in all 30 trials across the three objects. Fig.~\ref{fig:vaq_franka} shows representative trials, with more details included in the supplementary video.

\section{Discussion}

In this work, we propose value-preserving action quantization (VAQ), which extracts discrete representations of continuous actions by preserving the dynamics-induced state reachability of the system. Our analysis shows that accounting for robot dynamics in action quantization is necessary for preserving the value structure relevant to downstream decision making. By combining VAQ with Monte Carlo tree search (MCTS), we show that effective offline action quantization enables computationally efficient online policy learning, which enables zero-shot adaptation to novel environments and significantly improved data efficiency compared to state-of-the-art offline policy learning methods.

This work suggests a direction where the value of offline data is not to create unified end-to-end policy, but to construct unified action representations for efficient online learning. Current limitations include dependence on the coverage of the offline dataset and the tradeoff between quantization resolution and search efficiency, particularly for longer-horizon and higher-dimensional control. Future work will address these limitations through learning-based state representations, including latent-space models for stable long-horizon prediction~\cite{hafner_dream_2019,wu_daydreamer_2023}, and adaptive quantization.

% \addtolength{\textheight}{-2cm} 

\bibliographystyle{IEEEtran} %  IEEE style
\bibliography{references} 

@incollection{kocsis_bandit_2006,
	address = {Berlin, Heidelberg},
	title = {Bandit {Based} {Monte}-{Carlo} {Planning}},
	volume = {4212},
	isbn = {978-3-540-45375-8 978-3-540-46056-5},
	url = {http://link.springer.com/10.1007/11871842_29},
	doi = {10.1007/11871842_29},
	urldate = {2026-10-12},
	booktitle = {Machine {Learning}: {ECML} 2006},
	publisher = {Springer Berlin Heidelberg},
	author = {Kocsis, Levente and Szepesvári, Csaba},
	year = {2006},
	note = {Series Title: Lecture Notes in Computer Science},
	pages = {282--293},
}

@article{pfeifer_difference-aware_2026,
	title = {Difference-{Aware} {Retrieval} {Policies} for {Imitation} {Learning}},
	volume = {2026},
	url = {https://proceedings.iclr.cc/paper_files/paper/2026/hash/8c2ecd140ebe227a531edc9a0c3b69d4-Abstract-Conference.html?utm_source=chatgpt.com},
	language = {en},
	urldate = {2026-09-15},
	journal = {International Conference on Learning Representations},
	author = {Pfeifer, Quinn and Pronovost, Ethan and Shah, Paarth and Khetarpal, Khimya and Srinivasa, Siddhartha and Gupta, Abhishek},
	month = apr,
	year = {2026},
	pages = {87175--87198},
}

@article{he_demystifying_2026,
	title = {Demystifying {Robot} {Diffusion} {Policies}: {Action} {Memorization} and a {Simple} {Lookup} {Table} {Alternative}},
	volume = {2026},
	shorttitle = {Demystifying {Robot} {Diffusion} {Policies}},
	url = {https://proceedings.iclr.cc/paper_files/paper/2026/hash/122ea6470232ee5e79a2649243348005-Abstract-Conference.html?utm_source=chatgpt.com},
	language = {en},
	urldate = {2026-09-15},
	journal = {International Conference on Learning Representations},
	author = {He, Chengyang and Liu, Xu and Sznaier Camps, Gadiel and Bruno, Joseph and Sartoretti, Guillaume and Schwager, Mac},
	month = apr,
	year = {2026},
	pages = {10890--10917},
}

@inproceedings{pari_surprising_2022,
	title = {The {Surprising} {Effectiveness} of {Representation} {Learning} for {Visual} {Imitation}},
	isbn = {978-0-9923747-8-5},
	url = {http://www.roboticsproceedings.org/rss18/p010.pdf},
	doi = {10.15607/RSS.2022.XVIII.010},
	language = {en},
	urldate = {2026-09-15},
	booktitle = {Robotics: {Science} and {Systems} {XVIII}},
	publisher = {Robotics: Science and Systems Foundation},
	author = {Pari, Jyothish and Shafiullah, Nur Muhammad (Mahi) and Arunachalam, Sridhar Pandian and Pinto, Lerrel},
	month = jun,
	year = {2022},
}

@article{swiechowski_monte_2023,
	title = {Monte {Carlo} {Tree} {Search}: a review of recent modifications and applications},
	volume = {56},
	issn = {1573-7462},
	shorttitle = {Monte {Carlo} {Tree} {Search}},
	url = {https://doi.org/10.1007/s10462-022-10228-y},
	doi = {10.1007/s10462-022-10228-y},
	language = {en},
	number = {3},
	urldate = {2026-09-14},
	journal = {Artificial Intelligence Review},
	publisher = {Springer},
	author = {Świechowski, Maciej and Godlewski, Konrad and Sawicki, Bartosz and Mańdziuk, Jacek},
	month = mar,
	year = {2023},
	pages = {2497--2562},
}

@inproceedings{castellini_scalable_2023,
	title = {Scalable {Safe} {Policy} {Improvement} via {Monte} {Carlo} {Tree} {Search}},
	issn = {2640-3498},
	url = {https://proceedings.mlr.press/v202/castellini23a.html},
	language = {en},
	urldate = {2026-09-13},
	booktitle = {Proceedings of the 40th {International} {Conference} on {Machine} {Learning}},
	publisher = {PMLR},
	author = {Castellini, Alberto and Bianchi, Federico and Zorzi, Edoardo and Simão, Thiago D. and Farinelli, Alessandro and Spaan, Matthijs T. J.},
	month = jul,
	year = {2023},
	pages = {3732--3756},
}

@inproceedings{grill_monte-carlo_2020,
	title = {Monte-{Carlo} {Tree} {Search} as {Regularized} {Policy} {Optimization}},
	issn = {2640-3498},
	url = {https://proceedings.mlr.press/v119/grill20a.html},
	language = {en},
	urldate = {2026-09-13},
	booktitle = {Proceedings of the 37th {International} {Conference} on {Machine} {Learning}},
	publisher = {PMLR},
	author = {Grill, Jean-Bastien and Altché, Florent and Tang, Yunhao and Hubert, Thomas and Valko, Michal and Antonoglou, Ioannis and Munos, Remi},
	month = nov,
	year = {2020},
	pages = {3769--3778},
}

@inproceedings{hafner_dream_2019,
	title = {Dream to {Control}: {Learning} {Behaviors} by {Latent} {Imagination}},
	shorttitle = {Dream to {Control}},
	url = {https://openreview.net/forum?id=S1lOTC4tDS},
	language = {en},
	urldate = {2026-06-30},
	booktitle = {International {Conference} on {Learning} {Representations}},
	author = {Hafner, Danijar and Lillicrap, Timothy and Ba, Jimmy and Norouzi, Mohammad},
	month = sep,
	year = {2019},
}

@inproceedings{wu_daydreamer_2023,
	title = {{DayDreamer}: {World} {Models} for {Physical} {Robot} {Learning}},
	issn = {2640-3498},
	shorttitle = {{DayDreamer}},
	url = {https://proceedings.mlr.press/v205/wu23c.html},
	language = {en},
	urldate = {2026-05-29},
	booktitle = {Proceedings of {The} 6th {Conference} on {Robot} {Learning}},
	publisher = {PMLR},
	author = {Wu, Philipp and Escontrela, Alejandro and Hafner, Danijar and Abbeel, Pieter and Goldberg, Ken},
	month = mar,
	year = {2023},
	pages = {2226--2240},
}

@inproceedings{lathrop_model_2024,
	title = {Model {Predictive} {Trees}: {Sample}-{Efficient} {Receding} {Horizon} {Planning} with {Reusable} {Tree} {Search}},
	issn = {2153-0866},
	shorttitle = {Model {Predictive} {Trees}},
	url = {https://ieeexplore.ieee.org/abstract/document/10802673},
	doi = {10.1109/IROS58592.2024.10802673},
	urldate = {2026-05-28},
	booktitle = {2024 {IEEE}/{RSJ} {International} {Conference} on {Intelligent} {Robots} and {Systems} ({IROS})},
	author = {Lathrop, John and Rivière, Benjamin and Alindogan, Jedidiah and Chung, Soon-Jo},
	month = oct,
	year = {2024},
	note = {ISSN: 2153-0866},
	pages = {7671--7678},
}

@inproceedings{tavakoli_learning_2020,
	title = {Learning to {Represent} {Action} {Values} as a {Hypergraph} on the {Action} {Vertices}},
	url = {https://openreview.net/forum?id=Xv_s64FiXTv},
	language = {en},
	urldate = {2026-05-28},
	booktitle = {International {Conference} on {Learning} {Representations}},
	author = {Tavakoli, Arash and Fatemi, Mehdi and Kormushev, Petar},
	month = oct,
	year = {2020},
}

@inproceedings{chebotar_q-transformer_2023,
	title = {Q-{Transformer}: {Scalable} {Offline} {Reinforcement} {Learning} via {Autoregressive} {Q}-{Functions}},
	issn = {2640-3498},
	shorttitle = {Q-{Transformer}},
	url = {https://proceedings.mlr.press/v229/chebotar23a.html},
	language = {en},
	urldate = {2026-05-28},
	booktitle = {Proceedings of {The} 7th {Conference} on {Robot} {Learning}},
	publisher = {PMLR},
	author = {Chebotar, Yevgen and Vuong, Quan and Hausman, Karol and Xia, Fei and Lu, Yao and Irpan, Alex and Kumar, Aviral and Yu, Tianhe and Herzog, Alexander and Pertsch, Karl and Gopalakrishnan, Keerthana and Ibarz, Julian and Nachum, Ofir and Sontakke, Sumedh Anand and Salazar, Grecia and Tran, Huong T. and Peralta, Jodilyn and Tan, Clayton and Manjunath, Deeksha and Singh, Jaspiar and Zitkovich, Brianna and Jackson, Tomas and Rao, Kanishka and Finn, Chelsea and Levine, Sergey},
	month = dec,
	year = {2023},
	pages = {3909--3928},
}

@inproceedings{tessler_distributional_2019,
	title = {Distributional {Policy} {Optimization}: {An} {Alternative} {Approach} for {Continuous} {Control}},
	volume = {32},
	shorttitle = {Distributional {Policy} {Optimization}},
	url = {https://proceedings.neurips.cc/paper_files/paper/2019/hash/72da7fd6d1302c0a159f6436d01e9eb0-Abstract.html},
	urldate = {2026-05-28},
	booktitle = {Advances in {Neural} {Information} {Processing} {Systems}},
	publisher = {Curran Associates, Inc.},
	author = {Tessler, Chen and Tennenholtz, Guy and Mannor, Shie},
	year = {2019},
}

@inproceedings{seyde_solving_2022,
	title = {Solving {Continuous} {Control} via {Q}-learning},
	url = {https://openreview.net/forum?id=U5XOGxAgccS},
	language = {en},
	urldate = {2026-05-28},
	booktitle = {The {Eleventh} {International} {Conference} on {Learning} {Representations}},
	author = {Seyde, Tim and Werner, Peter and Schwarting, Wilko and Gilitschenski, Igor and Riedmiller, Martin and Rus, Daniela and Wulfmeier, Markus},
	month = sep,
	year = {2022},
}

@article{tang_discretizing_2020,
	title = {Discretizing {Continuous} {Action} {Space} for {On}-{Policy} {Optimization}},
	volume = {34},
	copyright = {Copyright (c) 2020 Association for the Advancement of Artificial Intelligence},
	issn = {2374-3468},
	url = {https://ojs.aaai.org/index.php/AAAI/article/view/6059},
	doi = {10.1609/aaai.v34i04.6059},
	language = {en},
	number = {04},
	urldate = {2026-05-28},
	journal = {Proceedings of the AAAI Conference on Artificial Intelligence},
	author = {Tang, Yunhao and Agrawal, Shipra},
	month = apr,
	year = {2020},
	pages = {5981--5988},
}

@inproceedings{tavakoli_action_2018,
	address = {New Orleans, Louisiana, USA},
	series = {{AAAI}'18/{IAAI}'18/{EAAI}'18},
	title = {Action branching architectures for deep reinforcement learning},
	isbn = {978-1-57735-800-8},
	url = {https://dl.acm.org/doi/10.5555/3504035.3504541},
	urldate = {2026-05-27},
	booktitle = {Proceedings of the {Thirty}-{Second} {AAAI} {Conference} on {Artificial} {Intelligence} and {Thirtieth} {Innovative} {Applications} of {Artificial} {Intelligence} {Conference} and {Eighth} {AAAI} {Symposium} on {Educational} {Advances} in {Artificial} {Intelligence}},
	publisher = {AAAI Press},
	author = {Tavakoli, Arash and Pardo, Fabio and Kormushev, Petar},
	month = feb,
	year = {2018},
	pages = {4131--4138},
}

@inproceedings{sigurdson_spectral_2026,
	title = {Spectral {Decomposition} of {Inverse} {Dynamics} for {Fast} {Exploration} in {Model}-{Based} {Manipulation}},
	url = {http://arxiv.org/abs/2603.27796},
	doi = {10.48550/arXiv.2603.27796},
	urldate = {2026-05-27},
	booktitle = {2026 {IEEE} {International} {Conference} on {Robotics} and {Automation} ({ICRA})},
	author = {Sigurdson, Solvin and Riviere, Benjamin and Burdick, Joel},
	month = mar,
	year = {2026},
	note = {arXiv:2603.27796 [cs.RO]},
}

@inproceedings{lee_behavior_2024,
	title = {Behavior {Generation} with {Latent} {Actions}},
	issn = {2640-3498},
	url = {https://proceedings.mlr.press/v235/lee24y.html},
	language = {en},
	urldate = {2026-05-27},
	booktitle = {Proceedings of the 41st {International} {Conference} on {Machine} {Learning}},
	publisher = {PMLR},
	author = {Lee, Seungjae and Wang, Yibin and Etukuru, Haritheja and Kim, H. Jin and Shafiullah, Nur Muhammad Mahi and Pinto, Lerrel},
	month = jul,
	year = {2024},
	pages = {26991--27008},
}

@inproceedings{dadashi_continuous_2022,
	title = {Continuous {Control} with {Action} {Quantization} from {Demonstrations}},
	issn = {2640-3498},
	url = {https://proceedings.mlr.press/v162/dadashi22a.html},
	language = {en},
	urldate = {2026-05-27},
	booktitle = {Proceedings of the 39th {International} {Conference} on {Machine} {Learning}},
	publisher = {PMLR},
	author = {Dadashi, Robert and Hussenot, Léonard and Vincent, Damien and Girgin, Sertan and Raichuk, Anton and Geist, Matthieu and Pietquin, Olivier},
	month = jun,
	year = {2022},
	pages = {4537--4557},
}

@inproceedings{canas_learning_2012,
	title = {Learning {Probability} {Measures} with respect to {Optimal} {Transport} {Metrics}},
	volume = {25},
	url = {https://proceedings.neurips.cc/paper/2012/hash/c54e7837e0cd0ced286cb5995327d1ab-Abstract.html},
	urldate = {2026-05-25},
	booktitle = {Advances in {Neural} {Information} {Processing} {Systems}},
	publisher = {Curran Associates, Inc.},
	author = {Canas, Guillermo and Rosasco, Lorenzo},
	year = {2012},
}

@inproceedings{van_den_oord_neural_2017,
	title = {Neural {Discrete} {Representation} {Learning}},
	volume = {30},
	url = {https://proceedings.neurips.cc/paper_files/paper/2017/file/7a98af17e63a0ac09ce2e96d03992fbc-Paper.pdf},
	booktitle = {Advances in {Neural} {Information} {Processing} {Systems}},
	publisher = {Curran Associates, Inc.},
	author = {van den Oord, Aaron and Vinyals, Oriol and kavukcuoglu, koray},
	editor = {Guyon, I. and Luxburg, U. Von and Bengio, S. and Wallach, H. and Fergus, R. and Vishwanathan, S. and Garnett, R.},
	year = {2017},
}

@article{schrittwieser_mastering_2020,
	title = {Mastering {Atari}, {Go}, chess and shogi by planning with a learned model},
	volume = {588},
	copyright = {2020 The Author(s), under exclusive licence to Springer Nature Limited},
	issn = {1476-4687},
	url = {https://www.nature.com/articles/s41586-020-03051-4},
	doi = {10.1038/s41586-020-03051-4},
	language = {en},
	number = {7839},
	urldate = {2026-04-25},
	journal = {Nature},
	publisher = {Nature Publishing Group},
	author = {Schrittwieser, Julian and Antonoglou, Ioannis and Hubert, Thomas and Simonyan, Karen and Sifre, Laurent and Schmitt, Simon and Guez, Arthur and Lockhart, Edward and Hassabis, Demis and Graepel, Thore and Lillicrap, Timothy and Silver, David},
	month = dec,
	year = {2020},
	pages = {604--609},
}

@inproceedings{wu_discrete_2025,
	title = {Discrete {Policy}: {Learning} {Disentangled} {Action} {Space} for {Multi}-{Task} {Robotic} {Manipulation}},
	shorttitle = {Discrete {Policy}},
	url = {https://ieeexplore.ieee.org/document/11127630/},
	doi = {10.1109/ICRA55743.2025.11127630},
	urldate = {2026-01-12},
	booktitle = {2025 {IEEE} {International} {Conference} on {Robotics} and {Automation} ({ICRA})},
	author = {Wu, Kun and Zhu, Yichen and Li, Jinming and Wen, Junjie and Liu, Ning and Xu, Zhiyuan and Tang, Jian},
	month = may,
	year = {2025},
	pages = {8811--8818},
}

@inproceedings{luo_action-quantized_2023,
	title = {Action-{Quantized} {Offline} {Reinforcement} {Learning} for {Robotic} {Skill} {Learning}},
	issn = {2640-3498},
	url = {https://proceedings.mlr.press/v229/luo23a.html},
	language = {en},
	urldate = {2026-01-12},
	booktitle = {Proceedings of {The} 7th {Conference} on {Robot} {Learning}},
	publisher = {PMLR},
	author = {Luo, Jianlan and Dong, Perry and Wu, Jeffrey and Kumar, Aviral and Geng, Xinyang and Levine, Sergey},
	month = dec,
	year = {2023},
	pages = {1348--1361},
}

@article{silver_mastering_2016,
	title = {Mastering the game of {Go} with deep neural networks and tree search},
	volume = {529},
	copyright = {2016 Springer Nature Limited},
	issn = {1476-4687},
	url = {https://www.nature.com/articles/nature16961},
	doi = {10.1038/nature16961},
	language = {en},
	number = {7587},
	urldate = {2026-01-07},
	journal = {Nature},
	publisher = {Nature Publishing Group},
	author = {Silver, David and Huang, Aja and Maddison, Chris J. and Guez, Arthur and Sifre, Laurent and van den Driessche, George and Schrittwieser, Julian and Antonoglou, Ioannis and Panneershelvam, Veda and Lanctot, Marc and Dieleman, Sander and Grewe, Dominik and Nham, John and Kalchbrenner, Nal and Sutskever, Ilya and Lillicrap, Timothy and Leach, Madeleine and Kavukcuoglu, Koray and Graepel, Thore and Hassabis, Demis},
	month = jan,
	year = {2016},
	pages = {484--489},
}

@article{silver_general_2018,
	title = {A general reinforcement learning algorithm that masters chess, shogi, and {Go} through self-play},
	volume = {362},
	url = {https://www.science.org/doi/10.1126/science.aar6404},
	doi = {10.1126/science.aar6404},
	number = {6419},
	urldate = {2026-01-07},
	journal = {Science},
	publisher = {American Association for the Advancement of Science},
	author = {Silver, David and Hubert, Thomas and Schrittwieser, Julian and Antonoglou, Ioannis and Lai, Matthew and Guez, Arthur and Lanctot, Marc and Sifre, Laurent and Kumaran, Dharshan and Graepel, Thore and Lillicrap, Timothy and Simonyan, Karen and Hassabis, Demis},
	month = dec,
	year = {2018},
	pages = {1140--1144},
}

@inproceedings{hubert_learning_2021,
	title = {Learning and {Planning} in {Complex} {Action} {Spaces}},
	issn = {2640-3498},
	url = {https://proceedings.mlr.press/v139/hubert21a.html},
	language = {en},
	urldate = {2025-12-12},
	booktitle = {Proceedings of the 38th {International} {Conference} on {Machine} {Learning}},
	publisher = {PMLR},
	author = {Hubert, Thomas and Schrittwieser, Julian and Antonoglou, Ioannis and Barekatain, Mohammadamin and Schmitt, Simon and Silver, David},
	month = jul,
	year = {2021},
	pages = {4476--4486},
}

@article{riviere_monte_2024,
	title = {Monte {Carlo} tree search with spectral expansion for planning with dynamical systems},
	volume = {9},
	url = {https://www.science.org/doi/10.1126/scirobotics.ado1010},
	doi = {10.1126/scirobotics.ado1010},
	number = {97},
	urldate = {2025-12-08},
	journal = {Science Robotics},
	publisher = {American Association for the Advancement of Science},
	author = {Rivière, Benjamin and Lathrop, John and Chung, Soon-Jo},
	month = dec,
	year = {2024},
	pages = {eado1010},
}

@article{chi_diffusion_2024,
	title = {Diffusion policy: {Visuomotor} policy learning via action diffusion},
	issn = {0278-3649},
	shorttitle = {Diffusion policy},
	url = {https://doi.org/10.1177/02783649241273668},
	doi = {10.1177/02783649241273668},
	language = {en},
	urldate = {2025-01-08},
	journal = {The International Journal of Robotics Research},
	publisher = {SAGE Publications Ltd STM},
	author = {Chi, Cheng and Xu, Zhenjia and Feng, Siyuan and Cousineau, Eric and Du, Yilun and Burchfiel, Benjamin and Tedrake, Russ and Song, Shuran},
	month = oct,
	year = {2024},
}

@book{villani_optimal_2009,
	address = {Berlin, Heidelberg},
	series = {Grundlehren der mathematischen {Wissenschaften}},
	title = {Optimal {Transport}},
	volume = {338},
	copyright = {http://www.springer.com/tdm},
	isbn = {978-3-540-71049-3 978-3-540-71050-9},
	url = {http://link.springer.com/10.1007/978-3-540-71050-9},
	doi = {10.1007/978-3-540-71050-9},
	urldate = {2025-01-30},
	publisher = {Springer},
	author = {Villani, Cédric},
	year = {2009},
}

\section*{APPENDIX}

\subsection*{Proof of value approximation bound}

The proof follows the same structure as Lemma 1 in~\cite{riviere_monte_2024}; however, we generalize the result from deterministic state transitions to probabilistic state transitions.

\begin{assumption} \label{assumption:reward}
    We assume that the reward function is $L_r$-Lipschitz with respect to state transition: 
    \begin{align}
        |r(s,a)-r(s,a')| \leq L_r \cdot W_1\big(p(\cdot|s,a),p(\cdot|s,a')\big).
    \end{align}
\end{assumption}

\begin{definition}[Bellman optimality operator]
    Given $\mathcal{M}=(\mathcal{S},\mathcal{A},p,r,\gamma)$, the Bellman optimality operator $\mathcal{T}$ is:
    \begin{align}
        (\mathcal{T}V)(s) = \max_{a\in\mathcal{A}} \left[r(s,a) + \gamma \mathbb{E}_{s^\prime\sim p(\cdot|s,a)}[V(s^\prime)]\right].
    \end{align}
\end{definition}

\begin{lemma}[Bellman optimality fixed point] \label{lemma:bellman_fixed_point}
    For the optimal value function $V^*$ of an MDP $\mathcal{M}$,
    \begin{align}
        (\mathcal{T}V^*)(s)=V^*(s).
    \end{align}
\end{lemma}

\begin{lemma}[Bellman contraction] \label{lemma:bellman_contraction}
    For any two value functions $V_1$ and $V_2$, the Bellman optimality operator satisfies
    \begin{align}
        \|\mathcal{T}V_1-\mathcal{T}V_2\|_\infty \leq \gamma\|V_1-V_2\|_\infty.
    \end{align}
\end{lemma}

\begin{lemma}[Kantorovich-Rubinstein inequality~\cite{villani_optimal_2009}] \label{lemma:kr}
    For any $L_f$-Lipschitz function $f$ and probability distributions $\mu_1,\mu_2$ with finite first moments,
    \begin{align}
        |\mathbb{E}_{\mu_1}[f]-\mathbb{E}_{\mu_2}[f]| \leq L_f W_1(\mu_1,\mu_2).
    \end{align}
\end{lemma}

\noindent\textbf{Proof of Theorem~\ref{theorem:value_bound}}
Let $\mathcal{T}_1$ and $\mathcal{T}_2$ denote the Bellman optimality operators of $\mathcal{M}_1$ and $\mathcal{M}_2$. By Lemma~\ref{lemma:bellman_fixed_point},
\begin{align}
    \|V^*_1-V^*_2\|_\infty = \|\mathcal{T}_1V^*_1-\mathcal{T}_2V^*_2\|_\infty.
\end{align} Adding and subtracting $\mathcal{T}_1V^*_2$, the triangle inequality gives
\begin{align}
    \|V^*_1-V^*_2\|_\infty \leq \|\mathcal{T}_1V^*_1-\mathcal{T}_1V^*_2\|_\infty + \|\mathcal{T}_1V^*_2-\mathcal{T}_2V^*_2\|_\infty. \nonumber
\end{align} Applying Lemma~\ref{lemma:bellman_contraction},
\begin{align}
    \|V^*_1 {-} V^*_2\|_\infty \leq \gamma\|V^*_1 {-} V^*_2\|_\infty + \|\mathcal{T}_1V^*_2 {-} \mathcal{T}_2V^*_2\|_\infty. \label{eq:og_bound}
\end{align}

To bound the term $\|\mathcal{T}_1V^*_2-\mathcal{T}_2V^*_2\|_\infty$ in (\ref{eq:og_bound}). We first define the one-step lookahead value under $V^*_2$ as
\begin{align}
    Q_{V^*_2}(s,a) = r(s,a) + \gamma \mathbb{E}_{s^\prime\sim p(\cdot|s,a)}[V^*_2(s^\prime)].
\end{align}
For any state $s\in\mathcal{S}$ and actions $a_1\in\mathcal{A}_1$ and $a_2\in\mathcal{A}_2$, denote $\mu_1=p(\cdot|s,a_1)$ and $\mu_2=p(\cdot|s,a_2)$. By Assumption~\ref{assumption:reward},
\begin{align}
    |r(s,a_1)-r(s,a_2)| \leq L_r W_1(\mu_1,\mu_2).
\end{align} By Lemma~\ref{lemma:kr},
\begin{gather}
    |\mathbb{E}_{\mu_1}V^*_2-\mathbb{E}_{\mu_2}V^*_2| \leq L_{V^*_2}W_1(\mu_1,\mu_2), \\
    |Q_{V^*_2}(s,a_1) {-} Q_{V^*_2}(s,a_2)| \leq (L_r+\gamma L_V)W_1(\mu_1,\mu_2). \label{eq:pairwise_q_bound}
\end{gather} 

To bound how much $\mathcal{T}_1V^*_2$ can exceed $\mathcal{T}_2V^*_2$, define
\begin{align}
    a^*_1 &= \argmax_{a_1\in\mathcal{A}_1}Q_{V^*_2}(s,a_1), \\
    \tilde{a}_2 &= \argmin_{a_2\in\mathcal{A}_2}W_1\big(p(\cdot|s,a^*_1),p(\cdot|s,a_2)\big),
\end{align} since $ (\mathcal{T}_1V^*_2)(s) {=} Q_{V^*_2}(s,a^*_1)$ and $(\mathcal{T}_2V^*_2)(s) {\geq} Q_{V^*_2}(s,\tilde{a}_2)$, the bound of (\ref{eq:pairwise_q_bound}) gives
\begin{align}
    & (\mathcal{T}_1V^*_2)(s)-(\mathcal{T}_2V^*_2)(s) \nonumber \\
    & \leq |Q_{V^*_2}(s,a^*_1)-Q_{V^*_2}(s,\tilde{a}_2)| \nonumber\\
    & \leq (L_r+\gamma L_V)\sup_{\mathcal{A}_1}\inf_{\mathcal{A}_2}W_1\big(p(\cdot|s,a_1),p(\cdot|s,a_2)\big).  \label{eq:bellman_bound_direction_1}
\end{align}
Exchanging the roles of $\mathcal{A}_1$ and $\mathcal{A}_2$ similarly bounds how much $\mathcal{T}_2V^*_2$ can exceed $\mathcal{T}_1V^*_2$:
\begin{align}
    & (\mathcal{T}_2V^*_2)(s) - (\mathcal{T}_1V^*_2)(s) \nonumber \\
    & \leq (L_r+\gamma L_V)\sup_{\mathcal{A}_2}\inf_{\mathcal{A}_1}W_1\big(p(\cdot|s,a_1),p(\cdot|s,a_2)\big). \label{eq:bellman_bound_direction_2}
\end{align}
Combining (\ref{eq:bellman_bound_direction_1}) and (\ref{eq:bellman_bound_direction_2}) leads to the Hausdorff distance,
\begin{align}
    & |(\mathcal{T}_1V^*_2)(s)-(\mathcal{T}_2V^*_2)(s)| \nonumber \\
    & \leq (L_r+\gamma L_V)d_{W_1}\big(R_{\mathcal{M}_1}(s),R_{\mathcal{M}_2}(s)\big).
\end{align} Taking the supremum over $s\in\mathcal{S}$ for both sides,
\begin{align}
    & \sup_{s\in\mathcal{S}}|(\mathcal{T}_1V^*_2)(s)-(\mathcal{T}_2V^*_2)(s)| \nonumber\\
    & \leq (L_r+\gamma L_V)\sup_{s\in\mathcal{S}}d_{W_1}\big(R_{\mathcal{M}_1}(s),R_{\mathcal{M}_2}(s)\big),
\end{align} which gives
\begin{align}
    \|\mathcal{T}_1V^*_2-\mathcal{T}_2V^*_2\|_\infty \leq (L_r+\gamma L_V)B(\mathcal{M}_1,\mathcal{M}_2). \label{eq:second_term_bound}
\end{align}
Substituting (\ref{eq:second_term_bound}) into (\ref{eq:og_bound}) gives
\begin{align}
    & \|V^*_1-V^*_2\|_\infty \nonumber \\
    & \leq \gamma\|V^*_1-V^*_2\|_\infty + (L_r+\gamma L_V) \cdot B(\mathcal{M}_1,\mathcal{M}_2).
\end{align} Therefore,
\begin{align}
    \|V^*_1-V^*_2\|_\infty \leq \frac{L_r+\gamma L_V}{1-\gamma}B(\mathcal{M}_1,\mathcal{M}_2),
\end{align}
which completes the proof.

\end{document}